\documentclass[sigconf]{acmart}
\usepackage{amsmath,amssymb,amsfonts}

\usepackage{algorithmic}
\usepackage{graphicx}
\usepackage{textcomp}
\usepackage{algorithm}
\usepackage{algorithmic}
\usepackage{booktabs}
\usepackage{amsfonts}
\usepackage{array}
\usepackage{calc}
\usepackage{marvosym}
\usepackage{amsthm}
\newtheorem{theorem}{Theorem}

\usepackage{xcolor}
\AtBeginDocument{%
  \providecommand\BibTeX{{%
    \normalfont B\kern-0.5em{\scshape i\kern-0.25em b}\kern-0.8em\TeX}}}

\copyrightyear{2024}
\acmYear{2024}
\setcopyright{acmlicensed}\acmConference[CIKM '24]{Proceedings of the 33rd
ACM International Conference on Information and Knowledge
Management}{October 21--25, 2024}{Boise, ID, USA}
\acmBooktitle{Proceedings of the 33rd ACM International Conference on
Information and Knowledge Management (CIKM '24), October 21--25, 2024,
Boise, ID, USA}
\acmDOI{10.1145/3627673.3680073}
\acmISBN{979-8-4007-0436-9/24/10}
\begin{document}

\title{Causal Interventional Prediction System for Robust and Explainable Effect Forecasting}


\author{Zhixuan Chu$^{1,2}$, Hui Ding$^{3}$, Guang Zeng$^{3}$, Shiyu Wang\textsuperscript{3\Letter}, Yiming Li$^{4}$}

\affiliation{%
  \institution{$^1$ The State Key Laboratory of Blockchain and Data Security, Zhejiang University}
  \institution{$^2$ Hangzhou High-Tech Zone (Binjiang) Institute of Blockchain and Data Security }
  \institution{$^3$Ant Group}
  \institution{$^4$Nanyang Technological University}
  \city{Hangzhou}
  \country{China}}
\email{zhixuanchu@zju.edu.cn,dinghui.ding@alibaba-inc.com,senhua.zg@antfin.com,weiming.wsy@antgroup.com,liyiming.tech@gmail.com}

% \author{Zhixuan Chu}
% \affiliation{%
%   \institution{The State Key Laboratory of Blockchain and Data Security, Zhejiang University}
%   \institution{Hangzhou High-Tech Zone (Binjiang) Institute of Blockchain and Data Security }
%   \city{Hangzhou}
%   \country{China}
% }
% \email{zhixuanchu@zju.edu.cn}

% \author{Hui Ding}
% \affiliation{%
%   \institution{Ant Group}
%   \city{Hangzhou}
% \email{dinghui.ding@alibaba-inc.com}

% \author{Guang Zeng}
% \affiliation{%
%   \institution{Ant Group}
%   \city{Hangzhou}
% \email{senhua.zg@antfin.com}

% \author{Shiyu Wang \textsuperscript{\Letter}}
% \affiliation{%
%   \institution{Ant Group}
%   \city{Hangzhou}
% \email{weiming.wsy@antgroup.com}

% \author{Yiming Li}
% \affiliation{%
%   \institution{Nanyang Technological University}
%   % \city{Charlottesville}
%   \country{Singapore}}
% \email{liyiming.tech@gmail.com}

\renewcommand{\shortauthors}{Chu, et al.}

\begin{abstract}
Although the widespread use of AI systems in today's world is growing, many current AI systems are found vulnerable due to hidden bias and missing information, especially in the most commonly used forecasting system. In this work, we explore the robustness and explainability of AI-based forecasting systems. We provide an in-depth analysis of the underlying causality involved in the effect prediction task and further establish a causal graph based on treatment, adjustment variable, confounder, and outcome. Correspondingly, we design a causal interventional prediction system (CIPS) based on a variational autoencoder and fully conditional specification of multiple imputations. Extensive results demonstrate the superiority of our system over state-of-the-art methods and show remarkable versatility and extensibility in practice. 
\end{abstract}






\maketitle

\section{Introduction}\label{Sec: Introduction}

The rapid development of Artificial Intelligence (AI) has enabled the deployment of various AI-based systems in practical applications. AI-based forecasting is one of the industry's most commonly used AI-based systems, which is to make predictions based on past and present data via AI techniques. However, many of the AI algorithms implemented nowadays fail to guarantee robustness and explainability in the face of complex industrial applications. For example, when working in crucial financial scenarios, the performance of the AI-based forecasting systems is always unreliable due to hidden bias and insurmountable missingness, which can result in catastrophic failures in the next set of business operations and cash flow management.

In this work, we take the marketing campaign effect forecasting \cite{chu2022hierarchical} as an example to explore the issues about the robustness and explainability involved in the AI-based forecasting system \cite{chu2024causal}. For example, Alipay, the top Fintech company, provides many marketing campaigns (e.g., interest-discounted credit) to support small or micro businesses. To reduce the liquidity risk and provide decisions for financing plans, the company needs to predict the effect of each marketing campaign and make financial planning every month in advance. According to the actual business requirements, the effect prediction for a marketing campaign faces four primary challenges.

Firstly, due to the enormous quantity and heterogeneity of customers, there exist various marketing campaigns that the strategy team designed, especially with varying degrees, duration, forms, etc. The feature distribution of customers under different marketing campaigns is unfixed. Therefore, there may be bias involved in the marketing campaign effect forecasting task due to the mismatch of customer feature distributions across the training data and target data. Secondly, it is unlikely that all the required variables for customers are available. Some variables (e.g., personal preferences, socioeconomic, willingness to demand) are hidden or unmeasured, so it is impossible to accurately estimate the marketing campaign effect based on the observed variables without further assumptions. For example, socioeconomic status is an important attribute, but it is hard to define and measure. A common practice is to rely on the so-called ``proxy variables''. For instance, instead of directly measuring socioeconomic status, we might get a proxy for it by knowing their zip code and income. Reasonably using proxy variables for hidden information is a crucial underpinning to effect prediction. Thirdly, besides the direct impact of a marketing campaign on the outcome, the marketing campaign always has an additional impact on the outcome via external factors, e.g., special holidays or festivals, shopping promotion activities, and so on. Although these external factors have nothing to do with the marketing campaign assignment, they might enlarge or shrink the marketing campaign's effect on the customers. Therefore, effectively utilizing such external information to help prediction is also very difficult. Finally, only part of the accessible variables are observed before the start of the campaign. As the marketing campaign comes into effect, more and more information (e.g., the detailed information of customers assigned to the marketing campaign) will be observed. How to obtain reliable predictions only based on limited information is crucial for the company's decision and operation.

To sum up the aforementioned challenges, the general methods based on vanilla regression models cannot effectively solve this kind of practical problem. According to the different roles of variables, in this paper, we provide an in-depth analysis of the underlying causality involved in the effect prediction task, and we further establish a causal graph based on marketing campaign assignment (treatment), external factors (adjustment variable), customer's features (confounder), and effect (outcome). Correspondingly, we design a \textbf{C}ausal \textbf{I}nterventional \textbf{P}rediction \textbf{S}ystem (CIPS) based on the variational autoencoder (VAE) and fully conditional specification of multiple imputation (FCSMI). Our system imputes the missing variables, infers the hidden confounders, and finally predicts the effect outcome under the causal intervention.

\section{Related Work}

\subsection{Structural Causal Model} The most commonly used framework in causal inference is the Structural Causal Model (SCM)~\cite{pearl2016causal}. SCM describes the causal mechanisms of a system where a set of variables and the causal relationship among them are modeled by a set of simultaneous structural equations. In an SCM, if a variable is the common cause of two variables, it is called the confounder. The confounder will induce a spurious correlation between these two variables to disturb the recognition of the causal effect between them. In addition, in many real applications, the confounder is always hidden or unobserved. In an SCM, if we want to deconfound two variables to calculate the true causal effect. The backdoor adjustment is the most direct method to eliminate the spurious correlation by approximating the ``physical intervention'' \cite{pearl2018book,yang2021causal}. In the SCM framework, besides the confounder, our work also involves treatment (the action that applies to a unit, i.e., marketing campaigns), outcome (the result of that treatment when applied to that unit, i.e., marketing campaign effect), and adjustment variables (predictive of outcomes but not associated with treatment, i.e., external factors).

\subsection{Variational Autoencoder}

Variational autoencoder (VAE) \cite{kingma2014semi} is a generative model that generates high-dimensional samples from a continuous space. In the probability model, the probability of data $X$ can be computed by:
\begin{equation}
    p(x) = \int p(x, z)dz =\int p(z)p(x|z)dz,
\end{equation}
where it is approximated by maximizing the evidence lower bound (ELBO):
\begin{equation}
    \log p_\theta(x) \geq \mathbb{E}_{z \sim q_\phi(z|x)} [\log p_\theta (x|z)] - \*{KL}(q_\phi(z|x) \parallel p(z)),
    \label{vae}
\end{equation}
where $p_\theta (x|z)$ denotes the decoder with parameters $\theta$, $q_\phi(z|x)$ is obtained by an encoder with parameters $\phi$, and $p(z)$ is a prior distribution, such as Gaussian distribution. $KL$ denotes the Kullback-Leibler divergence between two distributions. VAE can be used in learning individual-level causal effects from observational data \cite{louizos2017causal}. It utilizes latent variable modeling to simultaneously estimate the unknown latent space summarizing the confounders and the causal effect. However, it assumes all observed variables are the confounders. This assumption is unreasonable in the real case. Compared with the standard VAE method, we divide the variables into treatment features, confounders, and adjustment variables according to the real implication of each kind of observed variable. In addition, we apply this idea to the forecasting system instead of the original treatment effect estimation task.

\section{Preliminary}
 
\subsection{Causality Analysis}
\begin{figure}[t]
  \centering
  \includegraphics[width=1\columnwidth]{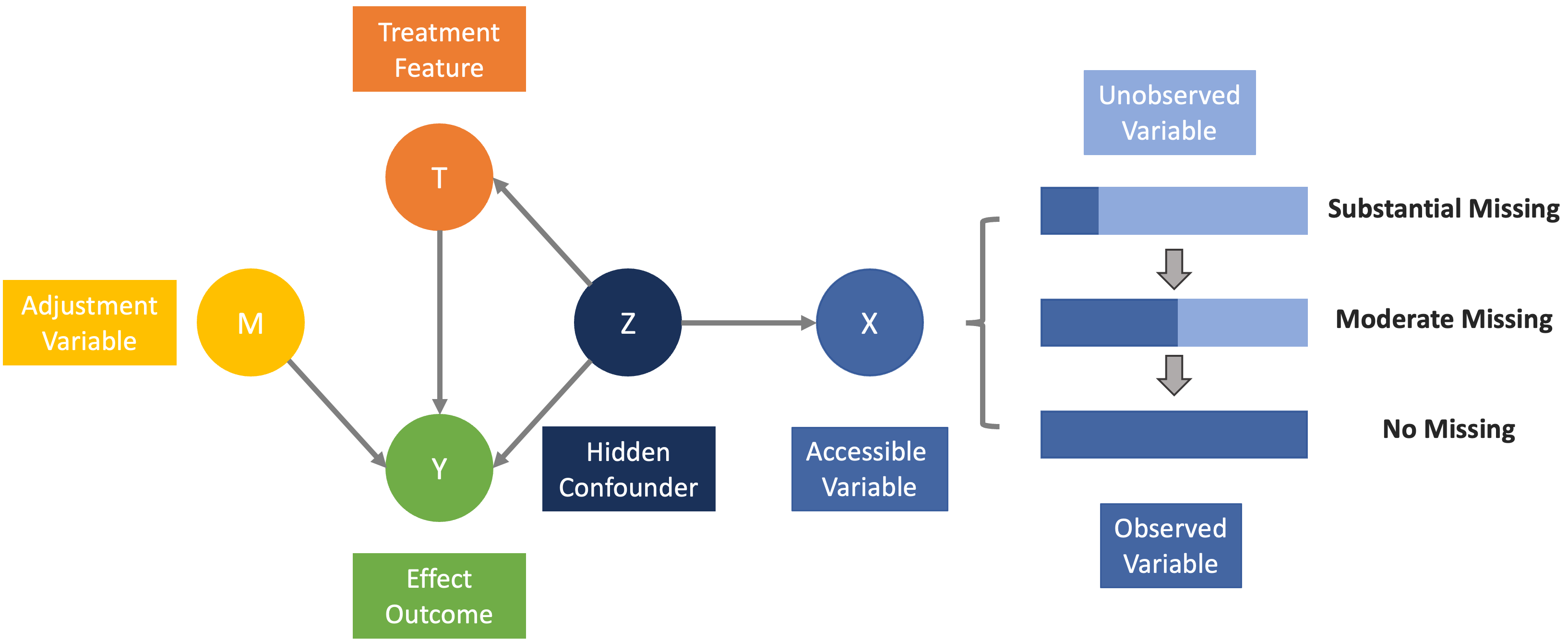}
\vspace{-6mm}
  \caption{\label{fig:framework} The causal graph involved in the CIPS.}
\vspace{-6mm}
\end{figure}

Taking the challenges mentioned into account, we incorporate the casual analysis into the effect prediction task. The causal analysis in Figure \ref{fig:framework} can help the model to tease out the relationship among different variables for predicting the effect and shed light on the impact of different variables.

The features of customer $X$ will act as the confounder, which influences both the marketing campaign assignment (treatment) and the effect of the marketing campaign (outcome) \cite{chu2024task}. The confounder will lead to the selection bias that different customers will be assigned to different types of marketing campaigns due to their features. Because the strategy team will continually assign various marketing campaigns to distinct customers according to certain rules, when performing the prediction of effects, we will face distribution shifts due to the perturbations in the pairs of assigned marketing campaigns and customers. Thus, the selection bias will induce a spurious correlation between marketing campaign assignment and effect outcome to disturb the recognition of the causal effect between them. Therefore, we aim to utilize the structural causal model to deconfound marketing campaigns and effect outcomes to calculate the true causal effect. Now, we are facing the second challenge that some confounders (personal preferences, socioeconomic, willingness to demand) are hidden or unmeasured. We hope to use the limited observed variables (``proxy variables'') $X$ to infer the hidden confounders $Z$. Of course, one of the promises of successfully inferring the hidden confounders is the existence of myriad proxy variables in the observed variables for hidden confounders. In addition, to increase the prediction accuracy given incomplete observed variables, we leverage the multiple imputation, that imputed datasets are sampled from their predictive distribution based on the observed data. The imputation procedure fully accounts for uncertainty in predicting the missing values by injecting appropriate variability into the multiple imputed values \cite{sterne2009multiple}. Besides the direct impact of the marketing campaign on the effect outcome, the marketing campaign always has an additional impact on the outcome via external factors, which might enlarge or shrink the marketing campaign's effect on customers. These variables can be referred to as adjustment variables predictive of effect outcomes but not associated with marketing campaign assignments. 

Considering all these components, we have established one complete causal graph for the effect prediction task, including the treatment, adjustment variable, hidden confounder, ``proxy'' accessible variable, incomplete observed variable, and effect outcome. We aim to apply the structural causal model based on a variational autoencoder to the established causal graph to identify the true causal effect and predict the unbiased effect outcome.
\begin{figure*}[t]
  \centering
  \includegraphics[width=1.7\columnwidth]{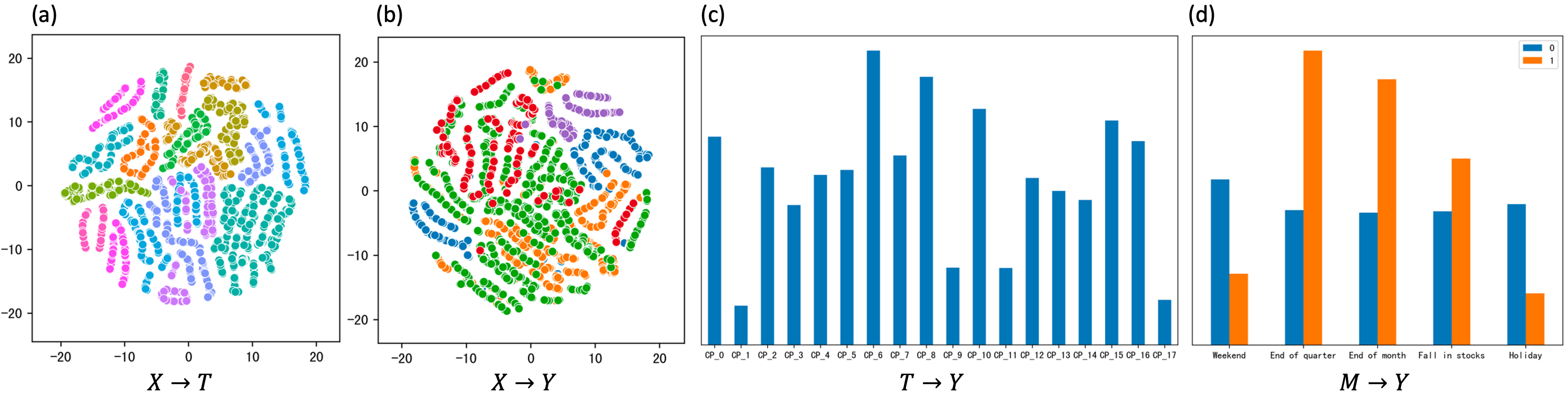}
  \vspace{-3mm}
  \caption{\label{fig:causal graph} The causal analysis in the real marketing campaign dataset including (a) $X \rightarrow T$ the distribution of marketing campaigns (treatment) in the customer feature  (confounder) space; (b) $X \rightarrow Y$ the distribution of different levels of outcome in the customer feature space (confounder); (c) $T \rightarrow Y$ the distribution of outcomes for different marketing campaigns (treatment); (d) $M \rightarrow Y$ the distribution of outcomes under different external factors (adjustment variable). Axis labels are omitted due to the nonpublic nature of the data.}
  \vspace{-3mm}
\end{figure*}

\subsection{Clarification on Treatment Effect Estimation}
The treatment effect estimation is arguably the most widely studied problem in causal inference \cite{rubin1974estimating,yao2021survey,chu2023continual,chu2023causal}. As a fundamental paradigm in causal inference, the Potential Outcomes Framework (POF) \cite{holland1986statistics, sekhon2008neyman} leverages factual outcomes and counterfactual outcomes to characterize treatment effects. In the conventional treatment effect estimation task, we aim to estimate the counterfactual outcome and then compare it with the observed outcome to calculate the treatment effect. For example, in the observational medication data, we only observe the factual outcome and never the counterfactual outcomes that would potentially have happened if they had chosen a different treatment option. Therefore, for the conventional treatment effect estimation task, the core purpose is to estimate the counterfactual outcome while reducing the selection bias. However, in our marketing campaign effect prediction task, we can easily get the true effect by a single group retrospective study, i.e., the performance comparison before and after a marketing campaign on the same customers. Thus, our \textbf{\emph{primary purpose is to predict the performance}} of the future marketing campaigns on incoming customers while controlling for the confounders (customer's features between marketing campaign assignment and effect outcome) and considering the external factors, \textbf{\emph{rather than estimating the counterfactual effect}} under other marketing campaigns that were not assigned. \textbf{\emph{Therefore, this marketing campaign effect prediction task is essentially a regression task, which is different from the treatment effect estimation task in causal inference. The traditional treatment effect estimation methodologies based on POF do not apply to our forecasting task}}.

\subsection{Problem Formulation}
Based on the above causality analysis, we can re-formulate this forecasting task. Denoting the effect outcome of subject $i$ for marketing campaign $t$ by $y_{i,t}$, our goal is to model the conditional distribution $P(y_{i,t} | x_i, m_i, do(t_i))$, where confounder $x_i$ denotes the feature variables of subject $i$ (e.g., age, incoming, job, credit record). Adjustment variable $m$ denotes the external factors (e.g., special holidays or festivals, shopping promotion activities, economic situations). The treatment variable $t$ denotes the assigned marketing campaign features, including the structured data (e.g., degrees, duration, forms, and date) and unstructured data (e.g., the text description of the campaign from the strategy team). 
For the observed feature variables $x_i \in \mathcal{R}^d$ of subjects, the number $d$ will gradually increase over time from the start of the marketing campaign. During the marketing campaign, the $d$ equals $D$, which is the maximum value of the dimension of accessible variables. Therefore, we need to (1) impute the incomplete observed variables to the complete accessible variables that are observed after the marketing campaigns; (2) infer the hidden confounders based on the imputed accessible variables; (3) consider the hidden confounders, adjustment variables, and treatment variables to predict the effect outcome.

\subsection{Evidence of Causal Relationship in Real Case}
Here, we use real data to demonstrate the existence of a causal relationship in the marketing campaign task. A detailed description of the dataset is provided in the ``Experiments'' section. As shown in (a) and (b) of Figure \ref{fig:causal graph}, there exists an obvious bias in the marketing campaign (treatment assignment) and outcomes. In (c), we can observe the distinct effect outcomes for different marketing strategies, controlling for the customers and external factors. In (d), external factors such as the weekend, end of the month, end of the quarter, holiday, and so on have a significant influence on the outcome, controlling for the customers and marketing strategy assignment. These figures can give an intuitive understanding of the causality implied in the real case. Some special things to note about this causal relationship are that it does not matter whether there is or is not an edge from $X$ to $T$ or $X$ to $Y$. This is because we intervene on $T$, and $Y$ is independent of $X$ given hidden confounder $Z$. They can be replaced by $Z \rightarrow T$ and $Z \rightarrow Y$.

\subsection{Causal Inference Identification}
 
The central question in the analysis of causal effect is the \textit{identification}: Can the controlled (post-intervention) distribution, be estimated from observed data governed by the pre-intervention distribution\cite{judea2010introduction}? Therefore, we propose the following theorem to prove the post-intervention distribution is identifiable.
\begin{theorem}
If the $p(y| X, M, T= t, Z)$ can be inferred from the observations of $(X,M,T,y)$, then $p\left( y | X, M, do( T= t)\right)$ is identifiable. 
\end{theorem}
 
\begin{proof}
\begin{equation}
\begin{split}
& \hspace{0.7cm} p\left( y | X, M, do( T= t)\right) \\
& \stackrel{(1)}{=} \int_{ Z} p\left( y| X, M,do( T= t), Z\right)p\left( Z| X, M,do( T= t)\right) \, d Z ,  \\
&\stackrel{(2)}{=}  \int_{ Z} p\left( y| X, M, T= t, Z\right)p\left( Z| X,M \right) \, d Z , \\
&\stackrel{(3)}{=}  \int_{ Z} p\left( y| X, M, T= t, Z\right)p\left( Z| X\right) \, d Z , \\
&\stackrel{(4)}{=}  \int_{Z}p\left(y|M, T=t,Z\right)p\left(Z|X\right) \, dZ, \\
& = \mathbb{E}_{p(z|x)} [p(y| M, T= t, Z)], \label{eq:ident}
\end{split}
\end{equation}
where equality (1) is by the Bayes'rule. Controlling for the hidden confounder $Z$, equality (2) can be taken by the rules of do-calculus. $M$ is the defined adjustment variable in the causal graph, which is only predictive of outcome $y$ and not related to marketing campaign assignment $T$. Thus, we can observe a collider ($M \rightarrow y \leftarrow Z$) in this causal graph, so unconditionally, M and Z are statistically independent of each other in equality (3). For equality (4), $y$ is independent of $X$ given $Z$, so we attain $p\left(y |X, M, T=t, Z\right) = p\left(y|M, T=t,Z\right)$. 

In conclusion,  the quantities in the final expression of Eq.~\eqref{eq:ident} are identifiable if we know the theorem's premise. In the next section, we will show how we estimate  $p(y| M, T= t, Z)$ from observations of $(X,M,T,y)$.
\end{proof}

\section{Methodology}

Now we will shed light on how to use observational data $(X, M, T, y)$ to infer the distribution $p(y| M, T= t, Z)$. We design a \textbf{C}ausal \textbf{I}nterventional \textbf{P}rediction \textbf{S}ystem (CIPS) based on variational autoencoder (VAE) and multiple imputation (MI) to predict the outcome by imputing the missing variables (to complete accessible variable), inferring the hidden confounders, controlling for adjustment variables, and mitigating the selection bias caused by confounders.

\subsection{Inference model}
To infer the target distribution, the normal neural network with nonlinear functions makes inference intractable, so we assume the true but intractable posterior takes on an approximate Gaussian form by variational inference. Therefore, we use VAE, including reference and generative models, to get a fixed form posterior approximation over the latent variables $Z$ in the inference model, and then reconstruct the input and predict the outcome in the generative model. The causal graph shows that the true posterior over $Z$ depends on $X, T$, and $y$, but it is statistically independent of $M$. Therefore we employ the following posterior approximation:
\begin{align*}
    q(\*z_i|\*x_i, \*t_i, y_i) =& \prod_{j=1}^{D_z}\mathcal{N}(\mu_{ij}={\!\mu_{i}}^{[j]}, \sigma^2_{ij} = {\!\sigma_{i}^2}^{[j]}) \\
    \!\mu_{i},\!\sigma^2_{i} =& g_1(\*t_i, \*x_i, y_i),
\end{align*}
where $[j]$ represents the $j$-th element in the vector and $g_1(\cdot)$ is a fully connected neural network with parameters $\phi_1$. 

For out-of-sample predictions, i.e., new subjects, we are required to complement the necessary information $y$ before inferring the distribution over $z$. For this reason, we will introduce an auxiliary distribution \cite{louizos2017causal} that will accommodate the inference model to infer the target distribution. Specifically, we will employ the following distributions for the effect outcomes $y$: 
\begin{align*}
q(y_i|\*x_i,\*m_i ,\*t_i) =& \mathcal{N}(\mu_i, \sigma_{i}^2), \\
\mu_i, \sigma_{i}^2 =& g_2(\*x_i,\*m_i,\*t_i)\label{eq:y_contt},
\end{align*}
where $g_2(\cdot)$ is a fully connected neural network with variational parameters $\phi_2$. 

In addition, when conducting the prediction for new subjects, we also face the problem of incomplete information for subjects. We propose to leverage the multiple imputation to complement the incomplete observed variable ($\*x_i \in \mathcal{R}^d$ where $d<D$) to the accessible variable ($\*x_i \in \mathcal{R}^D$). Multiple imputation for missing data is rooted in the Bayesian framework for statistical inference. Within this framework, the task of imputing missing values, $X_{mis}$, in a data set equates to a random draw of an imputed value from the posterior predictive distribution of the missing data, which we will denote as $k(X_{mis}|X_{obs},\theta)$, where $\theta$ is the vector of parameters that uniquely define this predictive distribution for the missing values \cite{liu2015multiple}. In the common situation where the missing data problem is multivariate, has an arbitrary pattern of missing values, and may include variables of different types (continuous, nominal, binary, ordinal), it is analytically difficult or impossible to evaluate the true expression for the joint posterior distribution. In such cases, we use fully conditional specification (FCS) \cite{van2006fully} that can approximate draws from the analytically intractable joint posterior \cite{liu2015multiple}. The FCS method uses an iterative sequence of draws from conditional distributions (linear regression for continuous, logistic regression for binary and ordinal, discriminant function method for nominal categorical) to simulate draws from the highly complex joint posterior distribution, including a set of variables with mixed distributional type. 

Instead of drawing the imputations from a pre-specified joint distribution, FCS imputations are generated sequentially by specifying an imputation model for each variable given a set of conditional densities, one model for each incomplete variable. Therefore, an appropriate regression model can be adopted for each variable with a great deal of flexibility \cite{lee2010multiple,van2011mice}. We assume $X$ is from the multivariate distribution $P(X|\theta)$ that can be completely specified by $\theta$, a vector of unknown parameters. The posterior distribution of $\theta$ is obtained by iterative sampling from conditional distributions of the form
\begin{align*}
p(&X^{[1]}|X^{[2]},X^{[3]},...,X^{[D]},\theta_1) \\
 &\vdots  \\
p(&X^{[D]}|X^{[1]},X^{[2]},...,X^{[D-1]},\theta_D)  
\end{align*}
FCS MI draws imputations based on the Gibbs sampler by iterating over the conditional densities and sequentially filling in the current draws of each variable \cite{carlin2008bayesian}. The $t$-th iteration of the Gibbs sampler is defined as
\begin{align*}
 \hat{\theta}_1^{(t)} &\sim p(\theta_1 |X_{obs}^{[1]}, X^{[2](t-1)}, X^{[3](t-1)},...,X^{[D](t-1)}), \\
  \hat{X}^{[1](t)} &\sim p( X^{[1]} | X_{obs}^{[1]}, X^{[2](t-1)}, X^{[3](t-1)},...,X^{[D](t-1)}, \hat{\theta}_1^{(t)} ), \\
  &\vdots  \\
 \hat{\theta}_D^{(t)} &\sim p(\theta_D |X^{[1](t-1)}, X^{[2](t-1)}, X^{[3](t-1)},...,X_{obs}^{[D]}), \\
 \hat{X}^{[D](t)} &\sim p( X^{[D]} | X^{[1](t-1)}, X^{[2](t-1)}, X^{[3](t-1)},...,X_{obs}^{[D]},\hat{\theta}_D^{(t)}),
\end{align*}
When the iteration reaches convergence, the current draws are referred to as the first imputed values. The iteration is repeated until the desired number of imputation datasets has been achieved.
 
\subsection{Generative model}
For now, instead of conditioning on observations $(X,M,T,y)$, we can condition on the $(Z,M,T,y)$ with latent variables $Z$ in the generative model. $q(\*z_i|\*x_i, \*t_i, y_i)$ obtained in the reference model is the approximation to the posterior of the generative model $p(\*x, \*z, \*t, y)$ due to the fact that the true posterior is intractable. Let the prior over the latent variables be the centered isotropic multivariate Gaussian:
\begin{align*}
p(\*z_i) = \prod_{j=1}^{D_z}\mathcal{N}(z_{ij}| 0, 1), 
\end{align*}
where $z_{ij}$ are mutually independent Gaussian variables. Because VAE is one generative model, each component of features relying on the latent variables $Z$ can be generated as \cite{kingma2013auto,rezende2014stochastic}:
\begin{align*}
p(\*x_i|\*z_i) = \prod_{j=1}^{D_x}p(x_{ij}| \*z_i);\quad
p(\*t_i|\*z_i) = \prod_{j=1}^{D_t}p(t_{ij}| \*z_i),
\end{align*}
where $p(\*x_{ij}|\*z_i)$ and $p(\*t_{ij}| \*z_i)$ are probability distributions for the observed covariates and treatment features. $D_x$, $D_z$, and $D_t$ are the dimension of $\*x$, $\*z$, and $\*t$. For a continuous outcome, we parametrize the probability distribution as a Gaussian with its mean and variance
\begin{align*}
p(y_i|\*m_i, \*t_i, \*z_i) &= \mathcal{N}(\mu=\hat{\mu}_i, \sigma^2=\hat{\sigma}_i^2) \\
\hat{\mu}_i, \hat{\sigma}_i^2 &= f(\*m_i, \*t_i, \*z_i) 
\end{align*}
where $f(\cdot)$ is a fully connected neural network with a nonlinear activation function.

So far, we have established the whole framework for our CIPS system. The KL divergence involved in Eq.~(\ref{vae}) can be computed and differentiated without estimation \cite{kingma2013auto}. Therefore, except for the FCS MI module, we can now form a single objective for the inference and generative networks, the variational lower bound of this graphical model \cite{louizos2017causal}:
\begin{align*}
\mathcal{L}  &= \sum_{i=1}^{N}\mathbb{E}_{q(\*z_i|\*x_i,\*t_i,y_i)}[\log p(\*x_i, \*t_i|\*z_i) \\
&+ \log p(y_i|\*m_i, \*t_i, \*z_i) \\
&+ \log p(\*z_i) - \log q(\*z_i|\*x_i,\*t_i,y_i)] \\
&+ \sum_{i=1}^{N}\big( \log q(y_i=y^*_i|\*m^*_i ,\*x^*_i, \*t^*_i)\big). 
\end{align*}
where $\*x_i,\*t^*_i$, and $y^*_i$ are the values in the training set for inferring the auxiliary distribution.

 \section{Experiments}
\subsection{Dataset}\label{dataset}

To demonstrate the effectiveness and practicability of our forecasting system in real applications, we conduct experiments on real datasets collected by a Fintech company. In real business scenarios, we predict the effects of two kinds of activities, i.e., \textit{company strategy} and \textit{marketing campaign}. Company strategy refers to the strategy with a consideration of the trade-off between the company's development and the assessment of risk, which has a great impact on the user's credit limit, financial service access restrictions, etc. Marketing campaigns are sets of strategic activities that promote a business's goal or objective. A marketing campaign could be used to promote a product, a service, or the brand as a whole, e.g., offering discounts, reducing interest, and providing coupons. As shown in Figure \ref{fig:correlation}, we can observe that the correlations among all variables in the \textit{company strategy} and \textit{marketing campaign} datasets are different. In addition, unlike the marketing campaign that only has a positive effect, the company strategy can lead to either positive or negative effects on the business. Thus, due to the significant difference between these two datasets, we train the models and conduct the predictions for these two kinds of activities, respectively.

For \textit{company strategy} and \textit{marketing campaign}, each dataset has four types of information, i.e., the features of the assigned activity (treatment) such as the launch time of the activity, the strength of discount, the duration of the activity, and so on; the effect of activity (outcome); the external factors (adjustment variable) such as the special holidays or festivals, the shopping promotion activities, economic situations, and so on; the covariates of the customer (observed confounders) such as age, income, job, education, credit rating and so on. Since detailed information about customers is obtained slowly over time, in order to simulate the real scene, we consider three scenarios to represent the different degrees of missingness among the incrementally available observational data, including the \emph{substantial missingness}, \emph{moderate missingness}, and \emph{no missingness}. The marketing campaign dataset contains 9,820 records and 51 features. The company strategy dataset contains 7,374 records and 52 features. We randomly sample $60\%$ and $20\%$ of units as the training and validation set and let the remaining be the test set. A detailed description of datasets is provided in Table \ref{table:dataset}.

\begin{table}[t!]
\centering
\vspace{-3mm}
\caption{The detailed description of datasets.}\label{table:dataset}
 \vspace{-3mm}
\scalebox{0.75}{
\begin{tabular} {m{2cm}<{\centering}cccc}

\bottomrule
Covariate & Causal type & Data type &  \textbf{Campaign} & \textbf{Strategy}\\
\hline
Activity description & Treatment & Unstructured  &  Text & Text \\
 & &Structured  &  3 & 3 \\
\hline
Customer features& Confounder & Structured & 35 &  35 \\
\hline
 & & Apriori-known  & 11 & 11 \\
External factors & Adjustment & Unknown (Categorical) & 1 & 2\\
 & &  Unknown (Continuous) & 1 &  1\\
\bottomrule
\end{tabular}}
\end{table}

\begin{table}[t!]
\centering
\vspace{-3mm}
  \caption{The mean and standard error of MAPE for three scenarios, i.e., no missing, moderate missing, and substantial missing.}
  \vspace{-3mm}
  \label{result}
  \scalebox{0.8}{
  \begin{tabular}{lccc}
    \toprule
   \textbf{Marketing campaign} &    No missing  &   Moderate missing&    Substantial missing \\
    \midrule
    Lasso  & $  14.48\pm8.40$ & $  18.25\pm9.81$ & $  24.54\pm12.65$ \\
    
    SVR  & $  8.62\pm5.52$ & $  13.45\pm7.98$ & $  20.34\pm10.83$ \\
    DNN  & $  7.82\pm 4.31$ & $  12.63\pm 6.40$ & $  19.47\pm 8.72$ \\
    XGBoost  & $  6.10\pm 4.28$ & $  10.61\pm5.32$ & $  17.26\pm7.58$  \\
    FKNNR  & $  6.15\pm4.42$ & $  9.37\pm5.63$ & $  15.47\pm7.83$ \\
    Transformer  & $  5.70\pm3.41$ & $  9.09\pm5.38$ & $  16.03\pm7.71$ \\

    \midrule
    CIPS (SMI) & $  \textbf{4.55} \pm \textbf{2.12}$ & $  8.12\pm 4.29$ & $  13.71\pm7.02$ \\
     CIPS (Ours) & $  \textbf{4.55} \pm \textbf{2.12}$ & $  \textbf{7.45}\pm \textbf{3.07}$ & $  \textbf{12.23}\pm \textbf{5.98}$ \\
    \bottomrule

    \toprule
    \textbf{Company strategy} &    No missing  &   Moderate missing &    Substantial missing\\
    \midrule
    
    Lasso  & $  28.85\pm 10.92$ & $  35.95\pm 17.38$ & $  42.66\pm 20.12$ \\
    
    SVR  & $  23.19\pm 9.62$ & $  32.81\pm14.22$ & $  39.74\pm 18.18$ \\
    DNN  & $  26.32\pm9.47$ & $  33.47\pm 15.62$ & $  40.99\pm 19.22$ \\
    XGBoost  & $  18.45\pm 8.55$ & $  26.12\pm 10.12$ & $  35.18\pm 16.31$  \\
    FKNNR  & $  14.18\pm 7.30$ & $  22.65\pm9.01$ & $  32.14\pm14.17$ \\
    Transformer  & $  8.44\pm 5.32$ & $  19.42\pm 9.31$ & $  34.68\pm 14.31$ \\

    \midrule
    CIPS (SMI) & $ \textbf{3.88} \pm \textbf{2.01}$ & $  10.17\pm 7.99$ & $  17.06\pm 12.15$ \\
     CIPS (Ours) & $ \textbf{3.88} \pm \textbf{2.01}$ & $ \textbf{7.58}\pm \textbf{4.24}$ & $ \textbf{11.98}\pm \textbf{6.45}$ \\
    \bottomrule

  \end{tabular}}
   \vspace{-3mm}
\end{table} 

\begin{figure}[t]
  \centering
  \includegraphics[width=0.7\columnwidth]{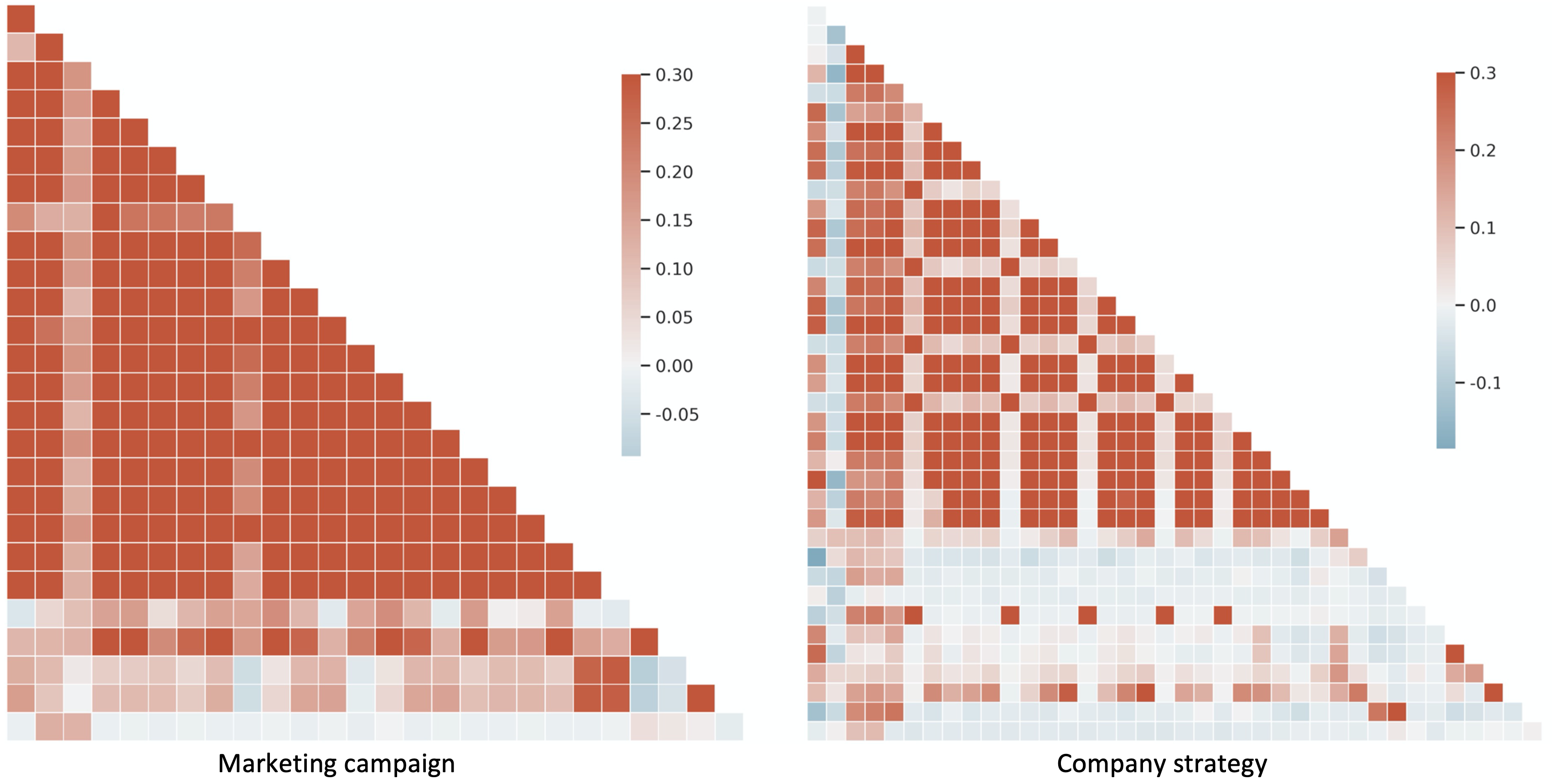}
    \vspace{-3mm}
  \caption{\label{fig:correlation} The correlation matrix for all variables in \textit{company strategy} and \textit{marketing campaign} datasets. Axis labels are omitted due to the nonpublic nature of the data.}
  \vspace{-3mm}
\end{figure}

\subsection{Baseline models}

We compare our model with baseline models on the real datasets with three scenarios, including substantial missingness, moderate missingness, and no missingness. As we mentioned in the ``Clarification on Treatment Effect Estimation'' section, the effect prediction task is essentially a regression task. The treatment effect estimation methodologies based on POF such as TARNET \cite{johansson2016learning}, CFR \cite{shalit2017estimating}, Dragonnet \cite{shi2019adapting}, GANITE \cite{yoon2018ganite}, FSRM\cite{chu2020matching}, IDRL \cite{chu2022learning}, GIAL \cite{chu2021graph}, CERL \cite{chu2023continual}, IGL \cite{sui2024invariant} are not applicable to our forecasting task. Therefore, we apply some classical regression models to this task. XGBoost \cite{chen2016xgboost} is a sparsity-aware algorithm for sparse data and a weighted quantile sketch for approximate tree learning. Lasso regression \cite{roth2004generalized} is an efficient linear regression with shrinkage (Lasso). FKNNR \cite{mailagaha2021generalized} is a fuzzy k-nearest neighbor regression model based on the usage of the Minkowski distance instead of the Euclidean distance. DNN \cite{larochelle2009exploring} is a deep multi-layer neural network with many levels of non-linearities, allowing them to represent highly nonlinear and highly-varying functions compactly. Support vector regression (SVR) \cite{awad2015support} is a type of support vector machine that supports linear and nonlinear regression. Transformer \cite{vaswani2017attention} is a deep learning model that uses self-attention, differentially weighting the importance of each element of the input data. For the moderate and substantial missing scenarios, we apply the single mean imputation (SMI) to these baseline models.

\subsection{Results}
We employ the mean absolute percentage error (MAPE) to measure the prediction accuracy of baseline models and our proposed CIPS. It usually expresses the accuracy as a ratio defined by the formula: $\text{MAPE} = \frac{100\%}{n} \sum_{i=1}^n |\frac{y_i-\hat{y}_i}{y_i}|$, where $\hat{y}_i$ is predicted value for subject $i$. As shown in Table \ref{result}, we provide the mean and standard error of MAPE for three scenarios, i.e., no missing, moderate missing, and substantial missing, in both datasets. Our CIPS consistently outperforms the state-of-the-art baseline methods with respect to both the mean and standard error of MAPE in all cases.

In addition, for all models, we can observe that the prediction accuracy gradually decreases as the degree of missingness increases. Compared to the drastic fluctuation in the baseline models, our CIPS is relatively robust to a high level of missingness. Another point worth noting is that, for most baseline models, the prediction accuracy for company strategy is significantly worse than that for marketing campaigns. To further figure out the reason why the baseline models have such a big gap, but the performance of our model is relatively consistent, we provide the distributions of customer features (confounder) grouped by treatment assignment and effect outcome in Figures \ref{fig:xt} and \ref{fig:xy}. There exists a worse selection bias caused by confounders in the company strategy dataset than that in the marketing campaign. It demonstrates that our model can more effectively deal with selection bias without sacrificing prediction accuracy. It also shows the necessity of causal analysis in practical prediction tasks. 

Besides, we want to explore the influence of multiple imputations. We conduct an ablation study, where the fully conditional specification multiple imputation in our model CIPS is replaced by a single mean imputation, i.e., CIPS (SMI). According to the last two rows in Table \ref{result}, the FCSMI remarkably improves the prediction accuracy compared with SMI. As shown in Figure \ref{fig:mistd}, we can observe the true value and predicted value with the standard error based on the multiple imputed datasets under the substantial missing scenario. The standard error in the company strategy is significantly greater than that in the marketing campaign dataset, which can also reflect the difficulty of predicting the effect outcome in the more biased company strategy dataset.

\begin{figure}[t]
  \centering
  \includegraphics[width=1\columnwidth]{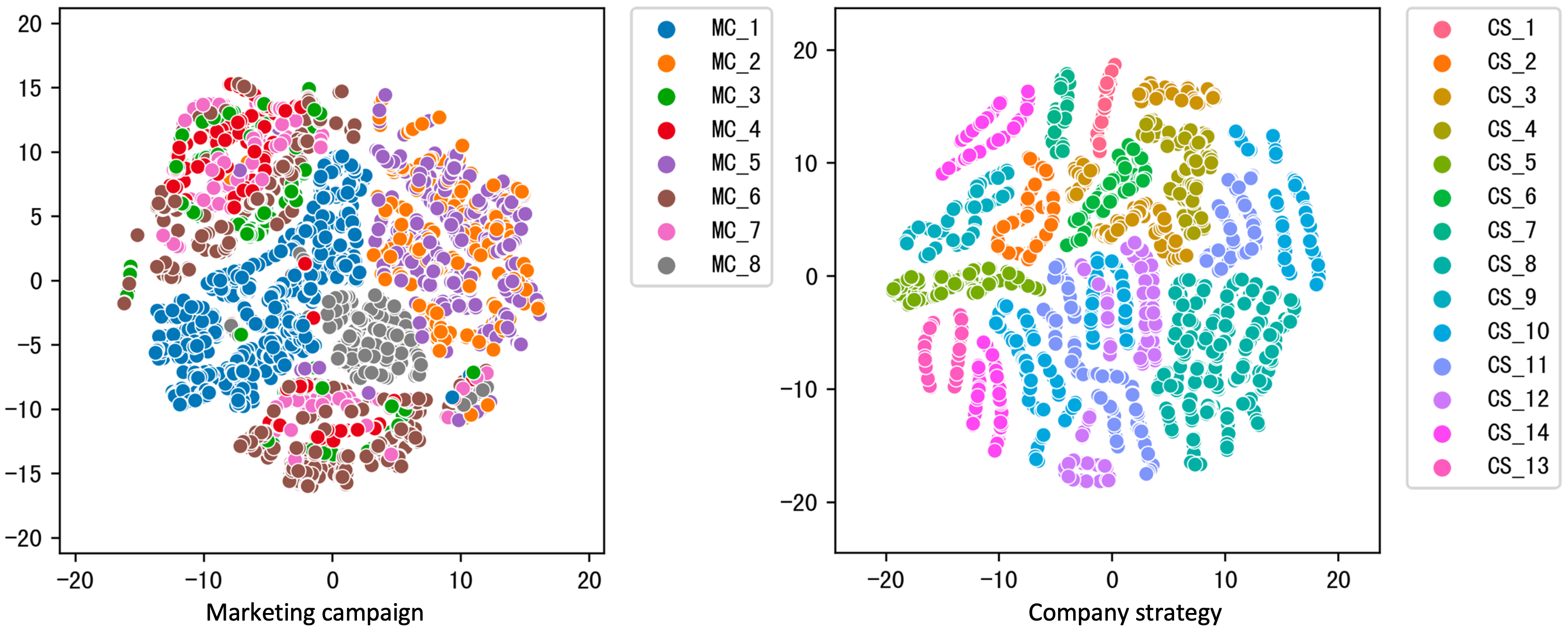}
  \vspace{-4mm}
  \caption{\label{fig:xt} Distribution of different \textbf{treatment assignment} (marketing campaign or company strategy) in \textbf{customer feature (confounder)} space. We can observe more distinct bias in the company strategy than in the marketing campaign.}   
  \vspace{-3mm}
\end{figure}

\begin{figure}[t]
  \centering
  \includegraphics[width=1\columnwidth]{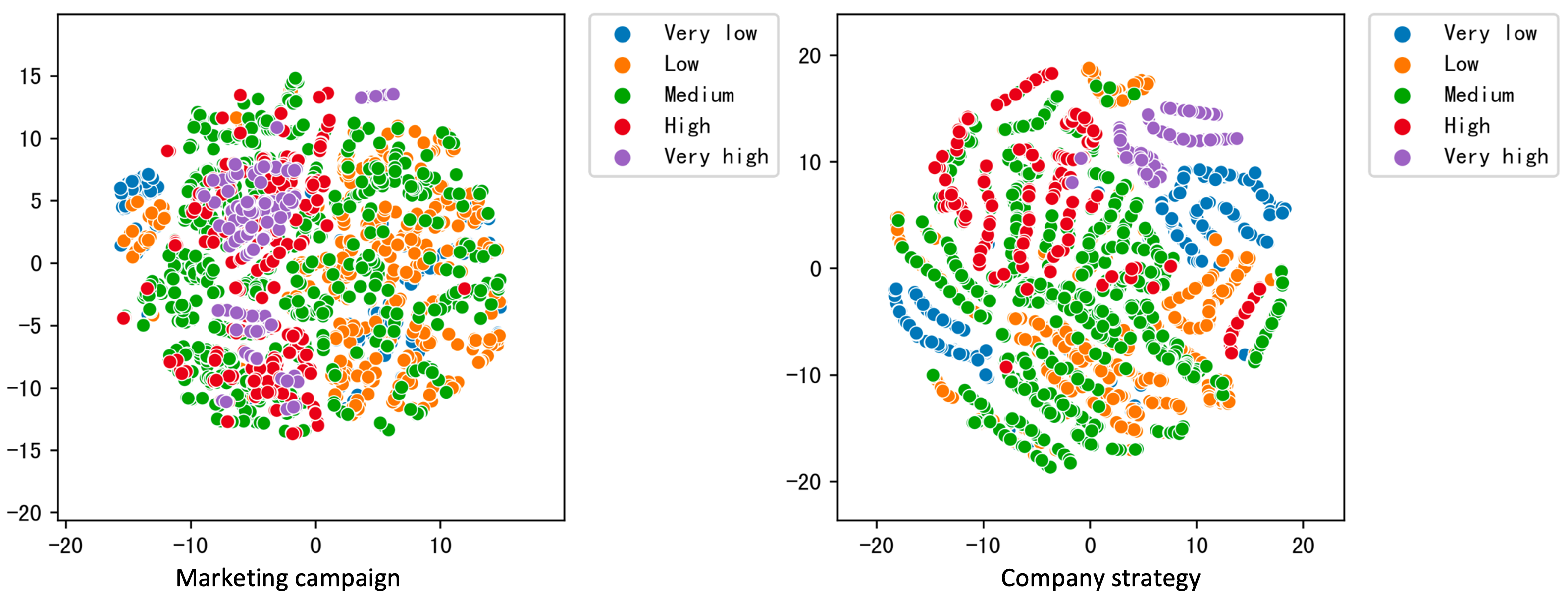}
  \vspace{-4mm}
  \caption{\label{fig:xy} Distribution of the different levels of \textbf{effect outcome} in \textbf{customer feature (confounder)} space. We can observe more distinct bias in the company strategy than in the marketing campaign.}
  \vspace{-3mm}
\end{figure}

\begin{figure}[t!]
  \centering
  \includegraphics[width=1\columnwidth]{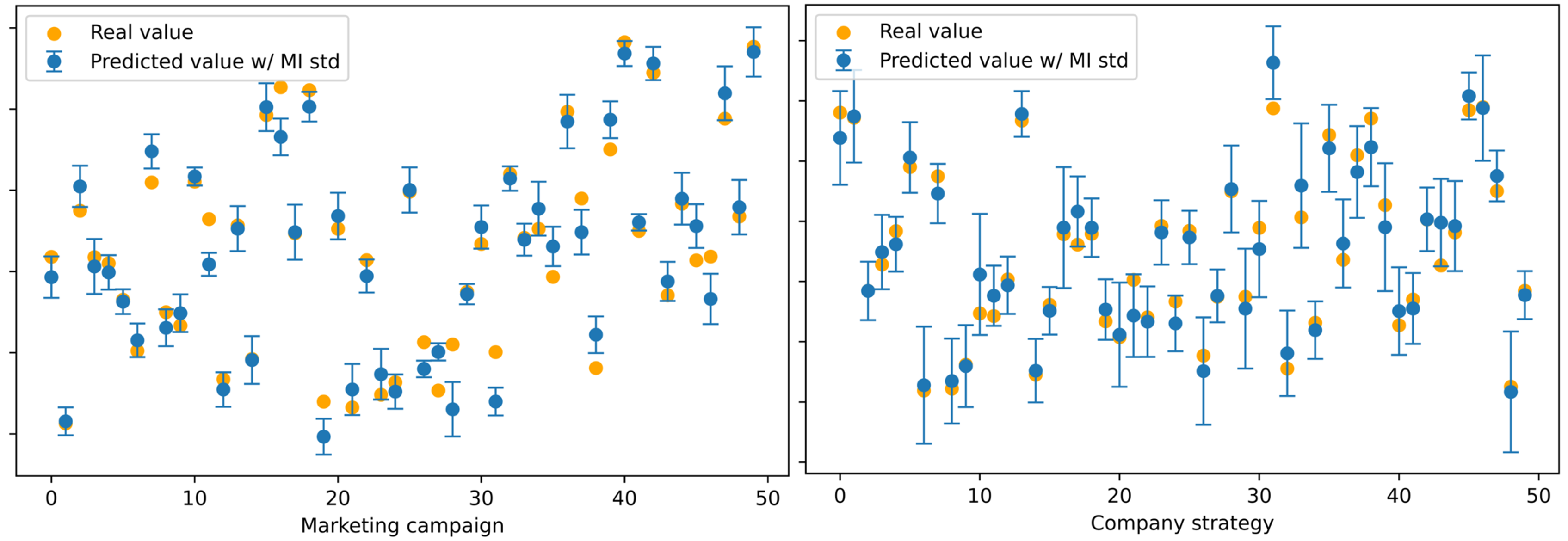}
    \vspace{-4mm}
  \caption{\label{fig:mistd} The true value and predicted value with the standard error based on the multiple imputed datasets under the substantial missing scenario. Axis labels are omitted due to the nonpublic nature of the data.} 
  \vspace{-3mm}
\end{figure}

\section{Deployment and Visualization}

\begin{figure}[t!]
  \centering
  \includegraphics[width=1\columnwidth]{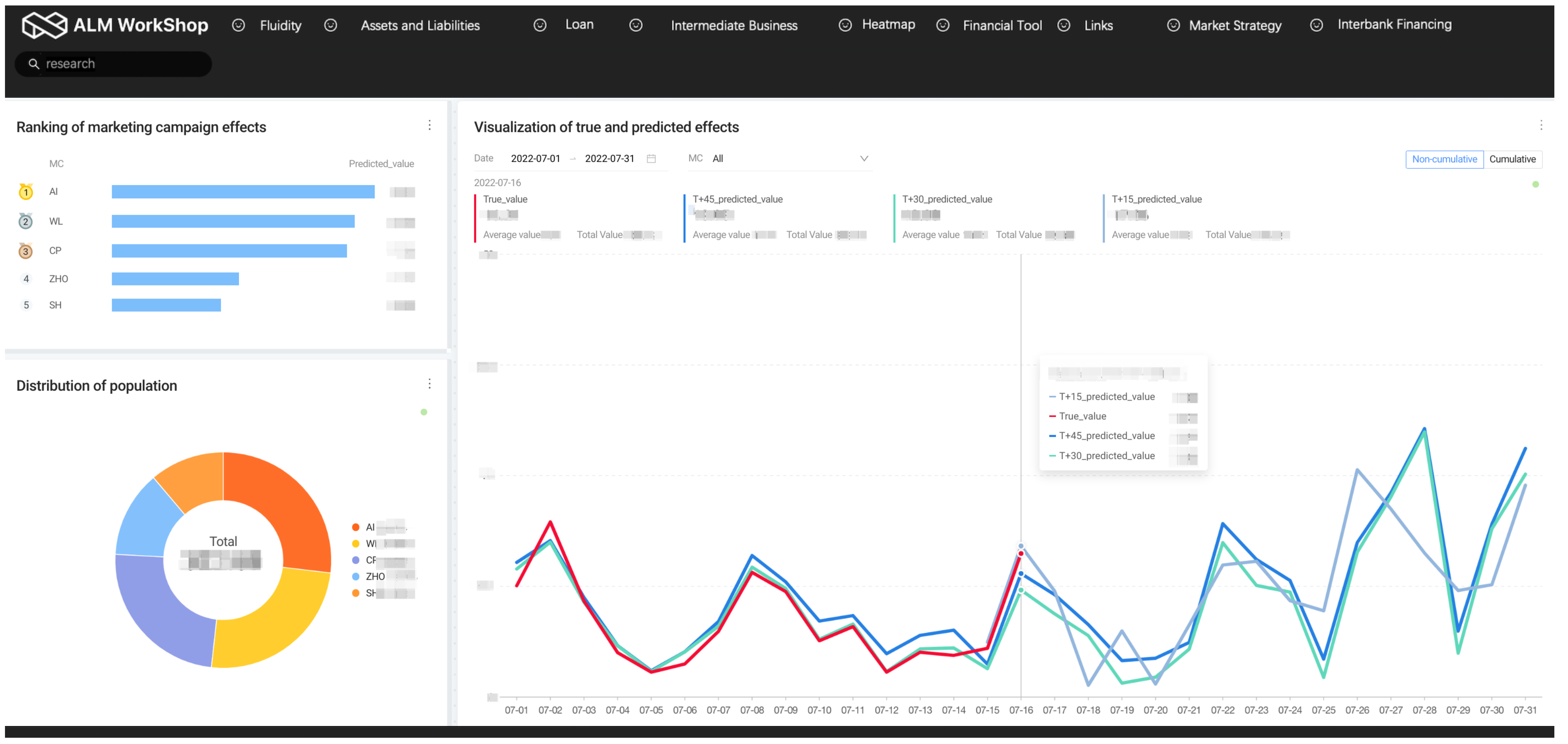}
   \vspace{-2mm}
  \caption{\label{fig:visualization}  The result display page of CIPS. It contains several different modules, such as the effect rankings of marketing campaigns, the distribution of customers in different marketing campaigns, and the visualization of true and predicted effects (including four types of lines, i.e., true value, T+45 predicted value for substantial missingness, T+30 predicted value for moderate missingness, and T+15 predicted value for no missingness).}

\end{figure}

\begin{figure}[t!]
  \centering
  \includegraphics[width=1\columnwidth]{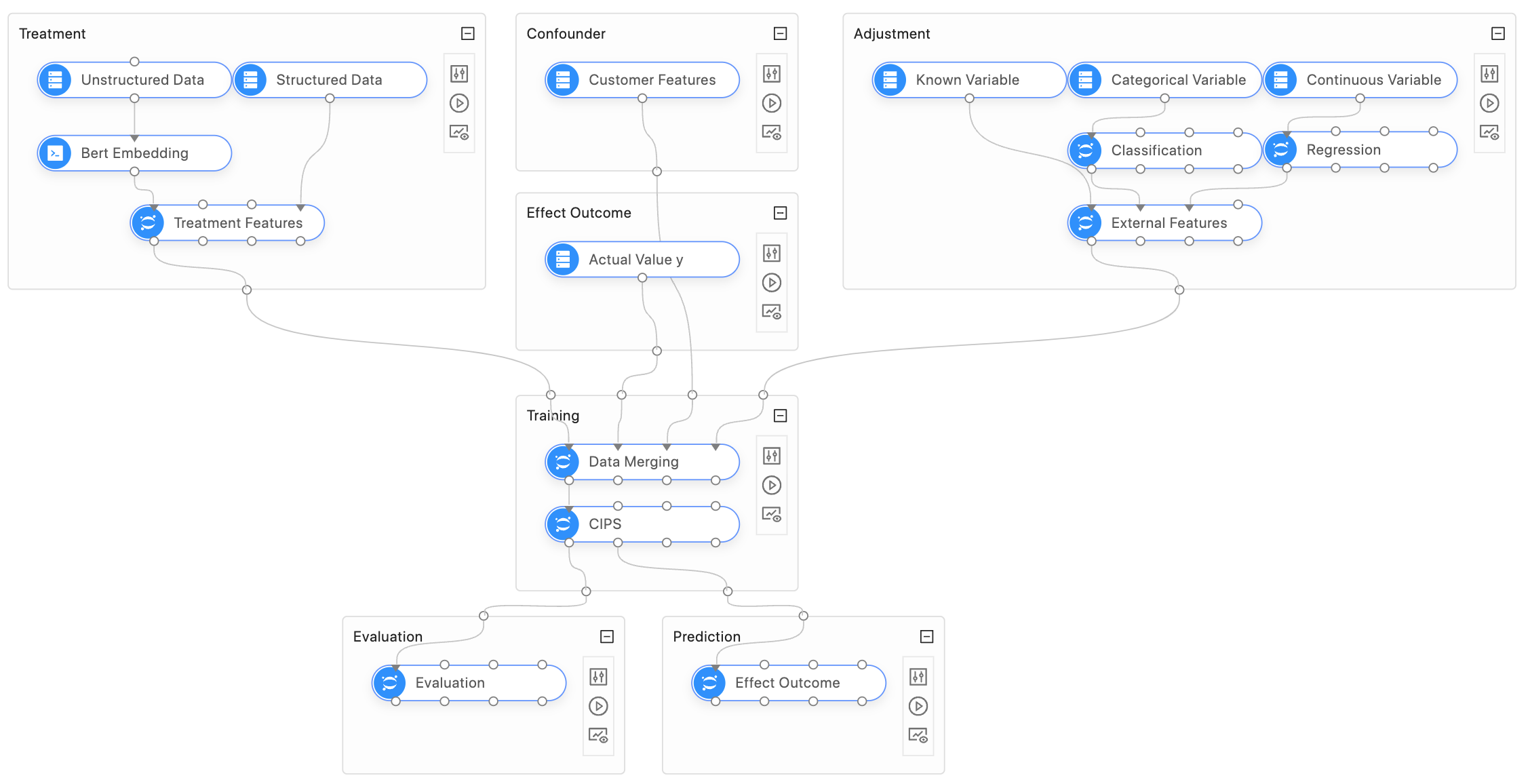}
  \caption{\label{fig:deployment} The deployment of CIPS. It contains four major modules, i.e., input, training, evaluation, and prediction.}
     \vspace{-4mm}
\end{figure}
 
This experiment is deployed on the internal algorithm platform. As shown in Figure \ref{fig:deployment}, it contains four major modules, i.e., input, training, evaluation, and prediction. The inputs of the experiment are from three types of variables, i.e., treatment, confounder, and adjustment variables. The treatment feature processing module can deal with both unstructured data and structured data. The unstructured data are texts provided by the strategy team to describe the marketing campaign or company strategy. The BERT \cite{devlin2018bert} is introduced to extract the feature vectors from texts. The confounder processing module is used to prepare the observed features of customers. The adjustment processing module is used to generate external factors. Some features are apriori-known future time-dependent inputs, such as special holidays or festivals, shopping promotion activities, and so on. They can be directly input into the model. For the unknown future features, we use pre-trained classification and regression models to predict the categorical and continuous variables, respectively. For example, a regression model predicts a big promotion's gross merchandise value (GMV). After preparing the data, the core of this system is our training module, which splices all the above-processed features and inputs them into our CIPS for training, evaluation, and prediction. In addition, our CIPS is deployed to a result exhibition system. The result display page in the exhibition system is shown in Figure \ref{fig:visualization}, which contains several different modules, such as the effect rankings of marketing campaigns, the distribution of customers in various marketing campaigns, and the visualization of true and predicted effects (including four types of lines, i.e., true value, T+45 predicted value for substantial missingness (45 days in advance), T+30 predicted value for moderate missingness (30 days in advance), and T+15 predicted value for no missingness (15 days in advance)). This marketing campaign effect forecasting system can effectively provide a robust and explainable prediction outcome and visually help guide the business planning and operation in advance.

\section{Conclusion}
To fulfill a robust and explainable AI-based forecasting task, we propose a causal interventional prediction system (CIPS) based on a variational autoencoder and a fully conditional specification of multiple imputations. The in-depth analysis of the underlying causality involved in the complex prediction task helps to improve the AI-based forecasting system concerning robustness and explainability. 

\newpage
\bibliographystyle{ACM-Reference-Format}
\bibliography{main}
\end{document}